\documentclass[letterpaper]{article}
\usepackage{times}
\usepackage{helvet}
\usepackage{courier}
\usepackage{arxiv}
\usepackage{verbatim}

\usepackage[utf8]{inputenc} % allow utf-8 input
\usepackage[T1]{fontenc}    % use 8-bit T1 fonts
\usepackage{hyperref}       % hyperlinks
\usepackage{url}            % simple URL typesetting
\usepackage{booktabs}       % professional-quality tables

\usepackage{subfig}
\usepackage{nicefrac}       % compact symbols for 1/2, etc.
\usepackage{microtype}      % microtypography
\usepackage{graphicx}
\usepackage{natbib}
\usepackage{doi}

\usepackage{multirow, rotating}

\usepackage{amsmath, amsthm, amssymb}
\usepackage{geometry}
\usepackage{booktabs}
\usepackage{array}
\usepackage{xcolor}
\usepackage{enumitem}
\usepackage{tikz}
\usepackage{algorithm}
\usepackage{algpseudocode}
\usepackage{caption}
\usetikzlibrary{positioning, arrows.meta, calc, shapes.geometric}

\usepackage[final, commandnameprefix=ifneeded, authormarkup=none,
            addedmarkup=colored, deletedmarkup=colored]{changes}

\definechangesauthor[color=blue!70!black]{new}
\definechangesauthor[color=red!70!black]{old}

\makeatletter
\let\Ch@add\added
\let\Ch@del\deleted
\renewcommand{\added}[1]{\Ch@add[id=new]{#1}}
\renewcommand{\deleted}[1]{\Ch@del[id=old]{#1}}
\renewcommand{\replaced}[2]{%
  \Ch@add[id=new]{#1}\ifbool{Changes@optiondraft}{ }{}\Ch@del[id=old]{#2}}

\newenvironment{addedblock}
  {\ifbool{Changes@optiondraft}{\color{blue!70!black}}{}}
  {}
\makeatother

\hypersetup{
  colorlinks = true,
  linkcolor  = blue!60!black,
  citecolor  = blue!60!black,
  urlcolor   = blue!60!black
}

\theoremstyle{plain}
\newtheorem{theorem}{Theorem}[section]
\newtheorem{lemma}[theorem]{Lemma}
\newtheorem{proposition}[theorem]{Proposition}

\theoremstyle{definition}
\newtheorem{definition}[theorem]{Definition}
\newtheorem{example}[theorem]{Example}

\theoremstyle{remark}
\newtheorem{remark}[theorem]{Remark}

\newcommand{\HTS}{\mathrm{H2S}}
\newcommand{\STH}{\mathrm{S2H}}
\newcommand{\wstar}{w^{*}}

\newcommand{\SigHG}{\Sigma_{\mathrm{HG}}}
\newcommand{\iso}{\cong}

\newcommand{\IsalHG}{\textsc{IsalHG}}
\newcommand{\IsalGraph}{\textsc{IsalGraph}}
\newcommand{\IsalSR}{\textsc{IsalSR}}
\newcommand{\cdll}{\mathcal{L}}
\newcommand{\ptr}[1]{p_{#1}}
\newcommand{\levi}{B}

\algrenewcommand\algorithmiccomment[1]{\hfill\textcolor{gray!70!black}{\(\triangleright\) #1}}

\title{Instruction Set and Language for Hypergraphs}

\author{ \href{https://orcid.org/0009-0001-2178-4647}{\includegraphics[scale=0.06]{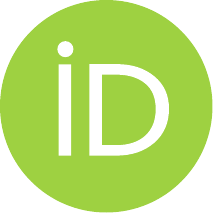}\hspace{1mm}Mario Pascual-Gonz\'alez} \\
    Department of Computer Languages and Computer Science\\
    University of M\'alaga\\
    Bulevar Louis Pasteur, 35\\
    29071 M\'alaga, Spain \\
	\texttt{mpascual@uma.es} \\
	\And
	\href{https://orcid.org/0000-0001-8231-5687}{\includegraphics[scale=0.06]{orcid.pdf}\hspace{1mm}Ezequiel L\'opez-Rubio}\thanks{Corresponding author. ITIS Software. Universidad de M\'alaga. C/ Arquitecto Francisco Peñalosa 18, 29010, Málaga, Spain} \\
	Department of Computer Languages and Computer Science\\
    University of M\'alaga\\
    Bulevar Louis Pasteur, 35\\
    29071 M\'alaga, Spain \\
	\texttt{ezeqlr@lcc.uma.es} \\
}

\renewcommand{\shorttitle}{Instruction Set and Language for Hypergraphs}

\hypersetup{
pdftitle={Instruction Set and Language for Hypergraphs},
pdfsubject={cs.DM, cs.DS},
pdfauthor={Mario Pascual-Gonz\'alez, Ezequiel L\'opez-Rubio},
pdfkeywords={hypergraph representation, hypergraph isomorphism, canonical form, instruction sequences, virtual machine, Levi graph}
}

\begin{document}
\maketitle

\begin{abstract}
We present \IsalHG{}, a method for representing the structure of any
finite, connected hypergraph of bounded hyperedge arity as a string over
a compact instruction alphabet $\SigHG$.  The encoding is executed by a
small virtual machine comprising a sparse hypergraph, a circular
doubly-linked list (CDLL) of node references, and $k$ traversal
pointers, where $k$ bounds the hyperedge arity.  Instructions either
move a pointer through the CDLL or insert a hyperedge, optionally
together with new nodes, into the hypergraph.  Every string over
$\SigHG$ decodes to a valid hypergraph; the alphabet is closed.  A greedy \emph{HypergraphToString} ($\HTS$) algorithm
encodes any connected hypergraph into a string; a backtracking variant
seeded at nodes of lexicographically maximal structural tuple produces a
\emph{canonical string} $\wstar_H$, which we \replaced{prove}{conjecture} to be a complete
isomorphism invariant.  Canonical-string equality then decides
hypergraph isomorphism natively, without the standard reduction to the
Levi incidence graph followed by a graph-isomorphism engine.  We verify
the round-trip property $\STH(\HTS(H)) \iso H$ on 150 connected random
uniform hypergraphs and on named combinatorial designs, and we benchmark
the canonical algorithm against the three practically available exact
baselines --- nauty, Traces, and bliss operating on the 2-coloured Levi
graph --- across a $(n, c)$ grid with ten seeds per cell.  All four
methods agree on every one of 600 isomorphism verdicts, \replaced{as the
completeness theorem requires}{consistent with
the completeness conjecture}.  On wall-clock time the Levi baselines
dominate every tested cell by three to five orders of magnitude
(geometric-mean ratio $311\times$ to $117{,}672\times$), which we report
as measured.  We
contribute the representation framework, \replaced{a proof that the
canonical string is complete}{a conjecture of canonical
completeness}, and the first native-versus-Levi benchmark for
hypergraph isomorphism.
\end{abstract}

% keywords can be removed

\section{Introduction\label{sec:introduction}}
Hypergraphs generalise graphs by allowing an edge to connect any number
of nodes rather than exactly two.  They are the natural model for group
interactions: co-authorship teams, chemical reactions, protein
complexes, social contact events, and legislative coalitions are all
sets of participants, not pairs \citep{berge1973graphs,
benson2018simplicial, battiston2020networks}.  As standard software for higher-order network
analysis has emerged
\citep{landry2023xgi, lotito2023hypergraphx}, the fundamental
representation question resurfaces at the hypergraph level: how should
the structure of a hypergraph be encoded so that structurally identical
objects can be recognised as such?  This is the hypergraph isomorphism
problem, and it underlies the classification of combinatorial designs
\citep{kaski2004steiner, colbourn2007handbook}, the deduplication of
hypergraph corpora, and the assessment of structural expressivity in
higher-order machine learning \citep{feng2024hypergraph}.

The established exact route to hypergraph isomorphism is a reduction.
The hypergraph $H$ is translated into its Levi incidence graph
$\levi(H)$: a bipartite graph with one vertex per node of $H$, one
vertex per hyperedge, and an edge whenever a node belongs to a hyperedge
\citep{berge1973graphs}.  Colouring the two vertex classes makes the
translation faithful, and a graph canonical-labelling engine --- nauty
or Traces \citep{mckay1981practical, mckay2014practical}, or bliss
\citep{junttila2007bliss} --- decides isomorphism on the reduced graph.
This pipeline is exact and mature; to our knowledge, it is the only exact hypergraph isomorphism procedure available as
working software: design-theory systems such as SageMath and GAP route
their incidence-structure isomorphism tests through nauty internally.
The pipeline is not \emph{native}.  The reduction inflates the
vertex set from $n$ to $n + m$, where $m$ is the number of hyperedges,
and the algorithm discards the hypergraph before any isomorphism reasoning
begins.  Native alternatives proposed so far are refinement-based
invariants in the Weisfeiler--Leman family, which are provably
incomplete \citep{feng2024hypergraph, zhang2025expressivity}, and
group-theoretic exact algorithms, which remain theoretical
\citep{babai2008hypergraph, neuen2022hypergraph,
schweitzer2019unifying}.

This paper introduces \IsalHG{} (Instruction Set and Language for
Hypergraphs), a native sequential representation of hypergraphs.  We encode a
hypergraph as a string over the instruction alphabet $\SigHG$
and execute it with a small virtual machine comprising a sparse
hypergraph, a circular doubly-linked list of node references, and $k$
traversal pointers, where $k$ bounds the hyperedge arity.  Every
string over $\SigHG$ decodes to a valid hypergraph.  A greedy
HypergraphToString algorithm ($\HTS$) encodes any connected hypergraph;
a backtracking variant seeded at nodes of maximal structural tuple
computes a canonical string $\wstar_H$, which we \replaced{prove}{conjecture} to be a
complete isomorphism invariant.  \replaced{Canonical-string equality therefore
decides}{Under this conjecture, canonical-string equality decides}
hypergraph isomorphism within the hypergraph domain.  \IsalHG{} is the third member of an
instruction-set representation family, after \IsalGraph{} for finite
simple graphs \citep{lopezrubio2025isalgraph, isalgraph2026} and
\IsalSR{} for the labelled directed acyclic graphs of symbolic
regression \citep{isalsr2026}.

We make three contributions.  First, we specify the alphabet, the
virtual machine, and the $\STH$/$\HTS$ algorithm pair, and motivate the design decisions.  Second, we \replaced{state and prove the round-trip
and canonical-completeness properties: the encoder and the interpreter
invert each other up to isomorphism, and the canonical string is a
complete isomorphism invariant}{state the round-trip and
canonical-completeness properties as explicit conjectures, argue both
directions informally, and defer the formal proofs to a dedicated
theoretical paper}.  Third, we report the first direct benchmark of
a native canonical-string method against the Levi route on hypergraphs:
on 150 connected random uniform hypergraphs spanning $n \in \{8, \dots,
25\}$ nodes, the round-trip property holds in every instance, all four
methods (\IsalHG{}, nauty, Traces, bliss) agree on all 600 isomorphism
verdicts, and the Levi baselines outperform the current canonical
algorithm by three to five orders of magnitude in wall-clock time.  We
report the runtime gap as measured; we characterise where the native encoding stands relative to
the reduction and make no claim of superiority over it.

The paper is organised as follows.  Section~\ref{sec:related} positions
\IsalHG{} against exact, approximate, and sequential prior work.
Section~\ref{sec:methodology} defines the alphabet, the virtual machine,
both conversion algorithms, and the \replaced{two properties and their
proofs}{conjectured properties}.
Section~\ref{sec:experiments} describes the data cohort and the
experimental protocol.  Section~\ref{sec:results} reports the round-trip,
agreement, and runtime results, which Section~\ref{sec:discussion}
discusses.  Section~\ref{sec:conclusion} concludes.

\section{Related work\label{sec:related}}
Prior work relevant to \IsalHG{} falls into four strands: the practical
exact graph canonical-labelling tools that reach hypergraphs through the
Levi reduction (\S\ref{sec:rel-exact-tools}), exact hypergraph
isomorphism algorithms that exist only in theory
(\S\ref{sec:rel-exact-theory}), Weisfeiler--Leman-style hypergraph
invariants that are native but incomplete (\S\ref{sec:rel-wl}), and
sequential encodings of combinatorial structures
(\S\ref{sec:rel-sequential}).  The first strand supplies our baselines;
the second and third explain why no other baseline is available; the
fourth contains the representation tradition \IsalHG{} extends.

\subsection{Exact graph canonical labelling and the Levi reduction}
\label{sec:rel-exact-tools}

Practical graph isomorphism is dominated by canonical-labelling tools
built on the individualisation--refinement (IR) paradigm: colour
refinement partitions the vertices, individualisation branches on the
cells the refinement cannot split, and a search tree over these branches
yields a canonical labelling.  nauty introduced the modern form of this
architecture \citep{mckay1981practical}; its sibling Traces replaced the
depth-first traversal with a breadth-first strategy and a different cell
selector, which pays off on highly symmetric inputs; both are maintained
and distributed together \citep{mckay2014practical}.  bliss refined the
same paradigm for large and sparse graphs \citep{junttila2007bliss}.
These engines decide isomorphism for graphs with millions of vertices in
practice; hard instances for all three are now well
characterised \citep{neuen2017benchmark}.  On the theory side, graph
isomorphism is decidable in quasipolynomial time
\citep{babai2016graph}, but the IR engines, with exponential worst
cases, remain the practical standard.

Hypergraphs enter this ecosystem through a reduction.  The Levi graph
$\levi(H)$ of a hypergraph $H$ with $n$ nodes and $m$ hyperedges is the
bipartite incidence graph on $n + m$ vertices in which each hyperedge
vertex is adjacent to the vertices of the nodes it contains
\citep{berge1973graphs}.  Colouring node-vertices and hyperedge-vertices
with two distinct colours makes the reduction faithful: two hypergraphs
are isomorphic exactly when their coloured Levi graphs are.  Every
software system we are aware of that decides hypergraph or
incidence-structure isomorphism exactly follows this route.  SageMath's
\texttt{IncidenceStructure.is\_isomorphic} and GAP's design-theory
packages invoke nauty internally, and the classification campaigns of
design theory --- for example the enumeration of the 11{,}084{,}874{,}829
Steiner triple systems of order 19 --- combine nauty-based isomorph
rejection with domain-specific invariants \citep{kaski2004steiner,
colbourn2007handbook}.  For hypergraph isomorphism, therefore, the
operational exact standard is the Levi reduction plus one of the three
engines above, and any native proposal must measure itself against that
pipeline.

\subsection{Exact hypergraph isomorphism in theory}
\label{sec:rel-exact-theory}

Hypergraph isomorphism is polynomial-time equivalent to graph
isomorphism: the Levi reduction maps hypergraphs to coloured graphs, and
graphs are hypergraphs of arity two.  The theoretical literature has
nevertheless sought algorithms whose complexity respects hypergraph
parameters instead of the inflated size $n + m$ of the reduction.
\citet{luks1999hypergraph} gave an algorithm exponential in the number
of nodes but polynomial in the number of hyperedges.
\citet{babai2008hypergraph} handle hypergraphs of bounded rank $k$ in
moderately exponential time $\exp(\tilde{O}(k^2 \sqrt{n}))$.
\citet{arvind2015colored} showed that colored hypergraph isomorphism is
fixed-parameter tractable in the maximum colour-class size.
\citet{neuen2022hypergraph} obtained the currently best bound,
$(n+m)^{O((\log d)^{c})}$ for groups with restricted composition
factors, while noting that the dependence on $m$ remains far from
optimal for hyperedge-rich inputs.  \citet{schweitzer2019unifying}
proposed a unifying canonisation framework over hereditarily finite
sets that canonises hypergraphs within the same asymptotic budget as
the group-theoretic algorithms.

The relevant point here is practical, not asymptotic: none of these five algorithms has a public implementation.
They are exact and, in spirit, native, but they cannot be run.  The
practical exact toolbox for hypergraph isomorphism thus contains
precisely the Levi reduction driven by nauty, Traces, or bliss;
Section~\ref{sec:experiments} benchmarks against those three engines
and no others.

\subsection{Weisfeiler--Leman invariants and hypergraph learning}
\label{sec:rel-wl}

A separate line of work computes isomorphism \emph{invariants} by
iterated colour refinement in the style of Weisfeiler and Leman
\citep{weisfeiler1968reduction}.  Such invariants are one-sided:
different values certify non-isomorphism, but equal values certify
nothing, and for every fixed refinement dimension $k$ there exist
non-isomorphic graph pairs the $k$-dimensional test cannot separate
\citep{cai1992optimal}.  On hypergraphs, colour refinement characterises
exactly the homomorphism counts of Berge-acyclic patterns
\citep{boker2019color}, which is strictly weaker than isomorphism.
\citet{feng2024hypergraph} introduced a hypergraph Weisfeiler--Leman
refinement together with the HIC tool, the only production-quality
native hypergraph fingerprint we are aware of; their own Figure~3
exhibits a pair of non-isomorphic hypergraphs the refinement collides,
and the authors do not characterise the failure family.
\citet{zhang2025expressivity} generalise the construction to a
$k$-dimensional hierarchy and prove it strict --- each level separates
pairs the previous level collapses --- at a cost that grows as
$O(h \cdot k \cdot n^{k+1})$, so completeness is unreachable at any
affordable level.  Kernel methods inherit the same ceiling:
\citet{bai2014hypergraph} build a hypergraph kernel from isomorphism
tests on the directed line graph, converting the hypergraph into a
graph before any comparison takes place.  These methods drive
successful hypergraph learning architectures, just as the
one-dimensional test bounds message-passing graph networks
\citep{xu2019powerful}, but they answer a different question than ours:
they trade completeness for tractability, whereas \IsalHG{} aims at a
complete invariant and accepts, for now, a large runtime cost.  We
therefore cite them as predecessors and do not benchmark against them
(Section~\ref{sec:setup} returns to this point).

\subsection{Sequential encodings of combinatorial structures}
\label{sec:rel-sequential}

Encoding structured objects as strings is an old and productive idea.
SMILES represents molecules as character strings and obtains canonical
forms through Morgan-style canonical atom ranking
\citep{weininger1988smiles}; SELFIES redesigned the alphabet so that
every string decodes to a valid molecule \citep{krenn2020selfies}, the
same closure property \IsalHG{} adopts as a design requirement.
Sequential graph generation models such as GraphRNN linearise graphs
into edge-event sequences for learning purposes, without canonicality
guarantees \citep{you2018graphrnn}.  Closest to our goal,
\citet{grzelak2021bigraph} define a canonical string encoding for pure
bigraphs, a different class of objects (process topologies with linked
interfaces) but the same logical aim of replacing structure comparison
by string comparison.

\IsalHG{} descends directly from two systems.  \IsalGraph{} represents
finite simple graphs as strings over a nine-character alphabet executed
by a virtual machine with a circular doubly-linked list and two
traversal pointers, and conjectures its exhaustive-backtracking
canonical string to be a complete graph invariant
\citep{lopezrubio2025isalgraph, isalgraph2026}.  \IsalSR{} transfers the
machine to the labelled expression DAGs of symbolic regression, adds a
two-tier labelled alphabet, and uses the canonical string to collapse
the $\Theta(k!)$ node-numbering redundancy of the search space
\citep{isalsr2026}.  \IsalHG{} generalises the machine along a third
axis: from edges over exactly two nodes with two pointers to hyperedges
over up to $k$ nodes with $k$ pointers.  We are not aware of any prior
work that represents hypergraphs as executable instruction sequences
serving as a canonical isomorphism invariant; the structurally closest
proposals are the bigraph encoding of \citet{grzelak2021bigraph} and the
two sibling systems above.

\section{Methodology\label{sec:methodology}}
This section defines the \IsalHG{} representation.  We fix notation and
scope (\S\ref{sec:preliminaries}), define the instruction alphabet and
the virtual machine that executes it (\S\ref{sec:instructions}), present
the string-to-hypergraph interpreter (\S\ref{sec:s2h}) and the greedy
hypergraph-to-string encoder (\S\ref{sec:h2s}), and construct the
canonical string together with the two properties we \replaced{prove and
then validate}{conjecture and
validate} empirically (\S\ref{sec:canonical}).

\subsection{Preliminaries and scope}
\label{sec:preliminaries}

A \emph{hypergraph} is a pair $H = (V, E)$ where $V$ is a finite set of
nodes and $E$ is a set of hyperedges, each hyperedge $e \in E$ being a
subset of $V$ with $|e| \geq 2$.  The \emph{arity} of a hyperedge is its
cardinality.  The \emph{primal graph} of $H$ is the simple graph on $V$
that joins two nodes whenever some hyperedge contains both; $H$ is
\emph{connected} when its primal graph is connected, and the distance
$d_H(u, v)$ between two nodes is their distance in the primal graph.
Two hypergraphs $H_1 = (V_1, E_1)$ and $H_2 = (V_2, E_2)$ are
\emph{isomorphic}, written $H_1 \iso H_2$, when a bijection
$\phi : V_1 \to V_2$ exists with $E_2 = \{ \phi(e) : e \in E_1 \}$.

Throughout the paper, input hypergraphs are finite, undirected,
connected, free of duplicate hyperedges, and of arity at most a fixed
parameter $k \geq 2$.  Directed hypergraphs, edge weights, and
disconnected inputs are outside the scope of this preprint; node and
hyperedge labels are supported by the alphabet but not exercised here
(Remark~\ref{rem:labels}).

\subsection{Instruction set and virtual machine}
\label{sec:instructions}

An \IsalHG{} string is executed by a virtual machine whose state is a
tuple
\[
  \mathcal{S} \;=\; (H, \cdll, \ptr{1}, \dots, \ptr{k}),
\]
where $H$ is the hypergraph under construction, $\cdll$ is a circular
doubly-linked list (CDLL) holding one reference per node of $H$, and
$\ptr{1}, \dots, \ptr{k}$ are $k$ traversal pointers into $\cdll$.  The
pointer count equals the arity bound $k$: an instruction that creates a
hyperedge designates each pre-existing member through one of the leading
pointers, so $k$ pointers suffice for hyperedges of arity up to $k$.
The initial state consists of a single node, a CDLL containing only that
node, and all $k$ pointers resting on it.

The instruction alphabet $\SigHG$ contains five token families,
summarised in Table~\ref{tab:instructions}.  Movement tokens $P_i$ and
$N_i$ advance and retreat pointer $\ptr{i}$ by one CDLL position.
Insertion tokens $V_{i,j}$ and $C_i$ create hyperedges: $V_{i,j}$
connects the $i$ nodes under $\ptr{1}, \dots, \ptr{i}$ to $j$ freshly
created nodes, which are spliced into $\cdll$ immediately after
$\ptr{1}$, while $C_i$ connects the $i$ nodes under
$\ptr{1}, \dots, \ptr{i}$ and creates no node.  Neither insertion token
moves a pointer.  $W$ does nothing.  Counting the admissible parameter
combinations gives
\[
  |\SigHG| \;=\; \underbrace{\tfrac{k(k-1)}{2}}_{V_{i,j}}
  \;+\; \underbrace{k}_{C_i}
  \;+\; \underbrace{k}_{P_i}
  \;+\; \underbrace{k}_{N_i}
  \;+\; \underbrace{1}_{W}
  \;=\; \tfrac{k(k-1)}{2} + 3k + 1 ,
\]
which is $13$ for the arity-3 hypergraphs of
Section~\ref{sec:experiments} and $76$ at the default cap $k = 10$.

\begin{table}[t]
  \centering
  \caption{The \IsalHG{} instruction set $\SigHG$.  The machine state is
  $(H, \cdll, \ptr{1}, \dots, \ptr{k})$; $\mathrm{val}(\ptr{i})$ denotes
  the node referenced by pointer $\ptr{i}$.  No instruction can fail:
  pointer aliasing shrinks the support of the created hyperedge, and a
  duplicate hyperedge turns $C_i$ into a no-op.}
  \label{tab:instructions}
  \begin{tabular}{@{}llp{8.6cm}@{}}
    \toprule
    Token & Constraints & Effect on the state \\
    \midrule
    $V_{i,j}$ & $1 \leq i, j \leq k{-}1$;\; $i{+}j \leq k$ &
      Create $j$ new nodes $u_1, \dots, u_j$; insert them into $\cdll$
      immediately after $\ptr{1}$; add the hyperedge
      $\{\mathrm{val}(\ptr{1}), \dots, \mathrm{val}(\ptr{i})\} \cup
      \{u_1, \dots, u_j\}$ to $H$.  Pointers do not move. \\
    $C_i$ & $1 \leq i \leq k$ &
      Add the hyperedge
      $\{\mathrm{val}(\ptr{1}), \dots, \mathrm{val}(\ptr{i})\}$ to $H$;
      skip if that hyperedge already exists.  Pointers do not move. \\
    $P_i$ & $1 \leq i \leq k$ &
      Advance $\ptr{i}$ one position forward in $\cdll$. \\
    $N_i$ & $1 \leq i \leq k$ &
      Retreat $\ptr{i}$ one position backward in $\cdll$. \\
    $W$   & --- & No operation. \\
    \bottomrule
  \end{tabular}
\end{table}

\begin{proposition}[Closure]
\label{prop:closure}
Every string $w \in \SigHG^{*}$ decodes to a valid hypergraph.
\end{proposition}

\begin{proof}
Every token is executable in every reachable state.  $\cdll$ is circular
and never empty, so $P_i$ and $N_i$ are always defined.  The pointers
always reference nodes of $H$, so the member sets built by $V_{i,j}$ and
$C_i$ are well defined; if several pointers alias the same node the
created hyperedge simply has smaller support, and $V_{i,j}$ always
contains at least one fresh node, so it never duplicates an existing
hyperedge.  $C_i$ skips duplicates by definition, and $W$ changes
nothing.  Hence execution is total and the final state contains a valid
hypergraph.
\end{proof}

Four decisions shape this alphabet.  First, hyperedge insertion is
atomic: one $V_{i,j}$ or $C_i$ token per hyperedge, so the string length
decomposes into a structural part fixed by $H$ and a traversal part
(Remark~\ref{rem:length}).  Second, the language is closed
(Proposition~\ref{prop:closure}): there are no syntactically or
semantically invalid strings, the property that makes molecular
alphabets such as SELFIES robust for generative use
\citep{krenn2020selfies} and that \IsalGraph{} adopted for graphs
\citep{isalgraph2026}.  Third, movements are unit steps on a circular
list, so the traversal cost of inserting a hyperedge equals the CDLL
distance between the pointers and the intended members; encoders that
keep this distance small produce short strings; cascade rule~C1 of
\S\ref{sec:h2s} minimises exactly this cost.  Fourth, new nodes enter
$\cdll$ next to $\ptr{1}$, keeping recently created nodes close to the
active region of the list and therefore cheap to reach while the
encoding of their neighbourhood completes.  The no-op $W$ is retained
for padding and for closure of the language under single-token edits,
following \IsalGraph{}.

\begin{remark}[Labelled extension]
\label{rem:labels}
$\SigHG$ extends to node- and hyperedge-labelled hypergraphs by
parameterising the insertion tokens with label identifiers,
$V[\ell_e; i; j; \ell_1, \dots, \ell_j]$ and $C[\ell_e; i]$, in the
two-tier style introduced by \IsalSR{} \citep{isalsr2026}.  The figures
in this paper render tokens in that serialised syntax with the single
trivial label $0$; the mathematical treatment of labelled hypergraphs is
deferred.
\end{remark}

\subsection{String-to-hypergraph conversion}
\label{sec:s2h}

The interpreter $\STH$ executes the tokens of a string sequentially from
the initial state and returns the final hypergraph.
Algorithm~\ref{alg:s2h} lists the dispatch;
Figure~\ref{fig:s2h-steps} traces it on the canonical string of the
Fano plane.  By Proposition~\ref{prop:closure} the interpreter is total:
it never rejects its input, which distinguishes $\SigHG$ from encodings
whose decoder must validate.

\begin{algorithm}[t]
\caption{$\STH(w, k)$: string-to-hypergraph interpreter.}
\label{alg:s2h}
\begin{algorithmic}[1]
\Require string $w \in \SigHG^{*}$, pointer count $k$
\Ensure hypergraph $H$
\State $H \gets$ single node $u_0$;\; $\cdll \gets [u_0]$;\;
       $\ptr{1}, \dots, \ptr{k} \gets u_0$
\For{each token $t$ of $w$, in order}
  \If{$t = P_i$}\; advance $\ptr{i}$ one position forward in $\cdll$
  \ElsIf{$t = N_i$}\; retreat $\ptr{i}$ one position backward in $\cdll$
  \ElsIf{$t = V_{i,j}$}
     \State create nodes $u_1, \dots, u_j$; insert them into $\cdll$ after $\ptr{1}$
     \State add hyperedge
       $\{\mathrm{val}(\ptr{1}), \dots, \mathrm{val}(\ptr{i})\}
        \cup \{u_1, \dots, u_j\}$ to $H$
  \ElsIf{$t = C_i$}
     \State $e \gets \{\mathrm{val}(\ptr{1}), \dots, \mathrm{val}(\ptr{i})\}$
     \If{$e \notin E(H)$}\; add hyperedge $e$ to $H$ \EndIf
  \ElsIf{$t = W$}\; do nothing
  \EndIf
\EndFor
\State \Return $H$
\end{algorithmic}
\end{algorithm}

\begin{example}
\label{exa:decode}
Let $k = 3$ and $w = V_{1,2}\,P_2\,C_2$.  Execution starts with the
single node $0$ and all pointers on it.  $V_{1,2}$ creates nodes $1$ and
$2$, splices them after $\ptr{1}$ so that $\cdll = [0, 1, 2]$, and adds
the arity-3 hyperedge $\{0, 1, 2\}$.  $P_2$ advances $\ptr{2}$ to node
$1$.  $C_2$ adds the arity-2 hyperedge
$\{\mathrm{val}(\ptr{1}), \mathrm{val}(\ptr{2})\} = \{0, 1\}$.  The
result is a hypergraph on three nodes with one triple and one pair.
\end{example}

\begin{figure}[t]
  \centering
  \includegraphics[width=\textwidth]{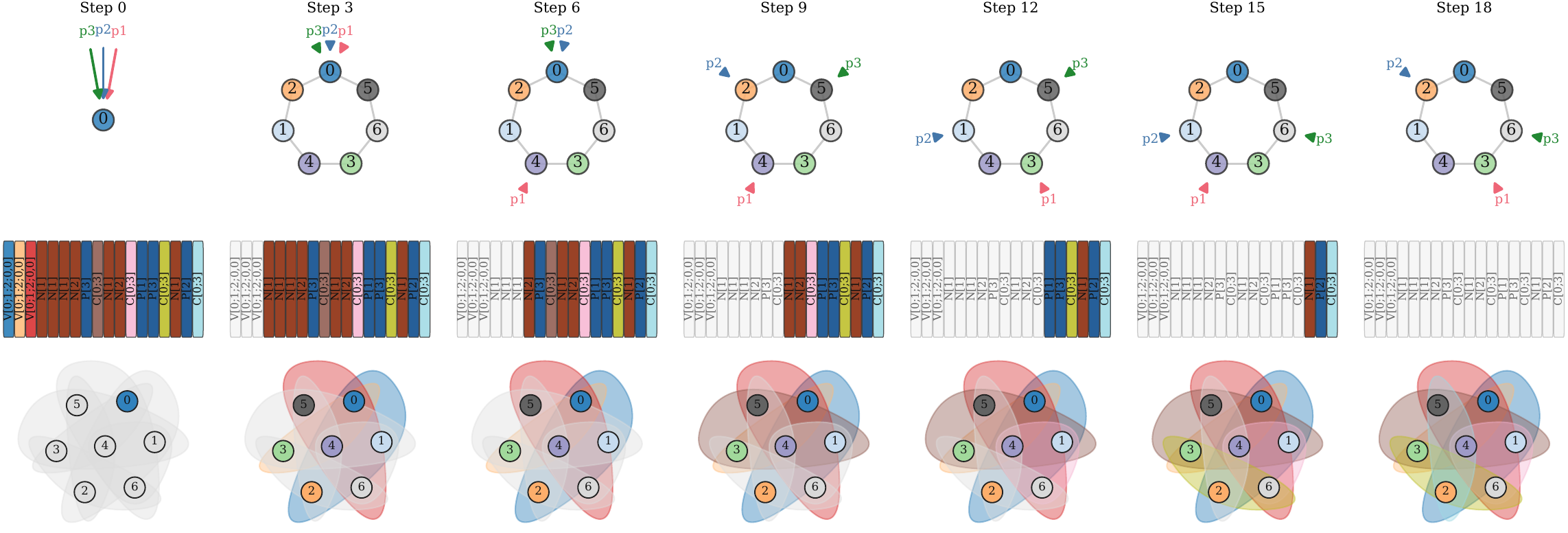}
  \caption{$\STH$ execution on the canonical string of the Fano plane
  STS(7) ($18$ tokens, $k = 3$), sampled every three instructions.  Top:
  the CDLL with the pointers $\ptr{1}$, $\ptr{2}$, $\ptr{3}$.  Middle:
  the token tape; pending instructions are highlighted and fade once
  executed.  Bottom: the output hypergraph, whose seven hyperedges
  materialise as the $V$/$C$ tokens execute.  Tokens are rendered in the
  serialised labelled syntax $V[\ell_e; i; j; \ell_1, \dots, \ell_j]$
  of Remark~\ref{rem:labels}, with all labels equal to the trivial
  $0$.}
  \label{fig:s2h-steps}
\end{figure}

\subsection{Hypergraph-to-string conversion}
\label{sec:h2s}

The greedy encoder $\HTS$ inverts the interpreter.  Given a connected
hypergraph $H$ and a start node $v_0$, it simulates the virtual machine
while maintaining a correspondence between the machine's nodes and the
nodes of $H$: the initial machine node corresponds to $v_0$, and every
hyperedge of $H$ is marked \emph{waiting} until a token inserts its
image.  At each step, a waiting hyperedge is \emph{insertable} when at
least one of its members already has a machine counterpart; its
materialised members must be brought under the leading pointers, and its
remaining members will be created fresh by the $V$ token itself.

The cost of a candidate is measured in pointer movements.  A
\emph{displacement tuple} $\delta = (\delta_1, \dots, \delta_k) \in
\mathbb{Z}^k$ moves each pointer $\ptr{l}$ by $|\delta_l|$ CDLL
positions, forward when $\delta_l > 0$ and backward when
$\delta_l < 0$, at total cost $\sum_l |\delta_l|$.  For each insertable
hyperedge the encoder computes the cost-minimal displacement tuple that
parks $\ptr{1}, \dots, \ptr{i}$ on its $i$ materialised members ---
candidate-driven search over the waiting hyperedges, which are far fewer
than the $(2|\cdll|+1)^k$ displacement tuples.  The winning candidate is
selected by a fixed cascade:

\begin{enumerate}[label=C\arabic*., leftmargin=2.2em, nosep]
  \item \emph{Cost.}  Minimise the total displacement
        $\sum_l |\delta_l|$; among equal totals, take the
        lexicographically smallest
        $\bigl(|\delta_1|, \dots, |\delta_k|,
        \delta_1, \dots, \delta_k\bigr)$.
  \item \emph{Kind.}  Prefer $V$ over $C$.
  \item \emph{Token.}  Take the lexicographically smallest $(i, j)$
        among $V$ candidates, or the smallest $i$ among $C$ candidates.
  \item \emph{Structure.}  Take the candidate hyperedge with the
        lexicographically smallest structural tuple $\eta(e)$
        (\S\ref{sec:canonical}).
  \item \emph{Backtracking.}  Any tie that survives C1--C4 spawns one
        branch per remaining candidate.  In addition, when a $V_{i,j}$
        token is emitted, the assignment of the $j$ new members of $e$
        to the $j$ insertion positions after $\ptr{1}$ is not determined
        by C1--C4; the encoder branches over these assignments as well.
        All branches are explored to completion and the
        lexicographically smallest completed string is kept.
\end{enumerate}

\noindent
The selected displacement is emitted as a block of movement tokens ---
retreats before advances, pointer index ascending, so that equal
displacements always serialise identically --- followed by the $V_{i,j}$
or $C_i$ token.  Algorithm~\ref{alg:h2s} summarises the loop, and
Figure~\ref{fig:h2s-steps} traces it on the Fano plane.  Connectivity
guarantees progress: while waiting hyperedges remain, at least one of
them touches the materialised region, so the loop terminates after
exactly $|E(H)|$ insertion tokens.

\begin{algorithm}[t]
\caption{Greedy $\HTS(H, v_0)$: hypergraph-to-string encoder.}
\label{alg:h2s}
\begin{algorithmic}[1]
\Require connected hypergraph $H$ of arity at most $k$; start node $v_0$
\Ensure string $w$ with $\STH(w, k) \iso H$
       (\replaced{Theorem}{Conjecture}~\ref{conj:roundtrip})
\State initialise the machine on a single node corresponding to $v_0$;
       mark every hyperedge of $H$ waiting
\While{waiting hyperedges remain}
  \For{each insertable waiting hyperedge $e$}
    \State compute the cost-minimal displacement $\delta(e)$ parking
           $\ptr{1}, \dots, \ptr{i}$ on the materialised members of $e$
  \EndFor
  \State select $(e, \delta(e))$ by the cascade C1--C5
  \State emit the movement block of $\delta(e)$; move the pointers
  \State emit $V_{i,j}$ if $e$ has $j \geq 1$ unmaterialised members,
         else $C_i$; execute it; extend the node correspondence; unmark
         $e$
\EndWhile
\State \Return the emitted token sequence
\end{algorithmic}
\end{algorithm}

\begin{figure}[t]
  \centering
  \includegraphics[width=\textwidth]{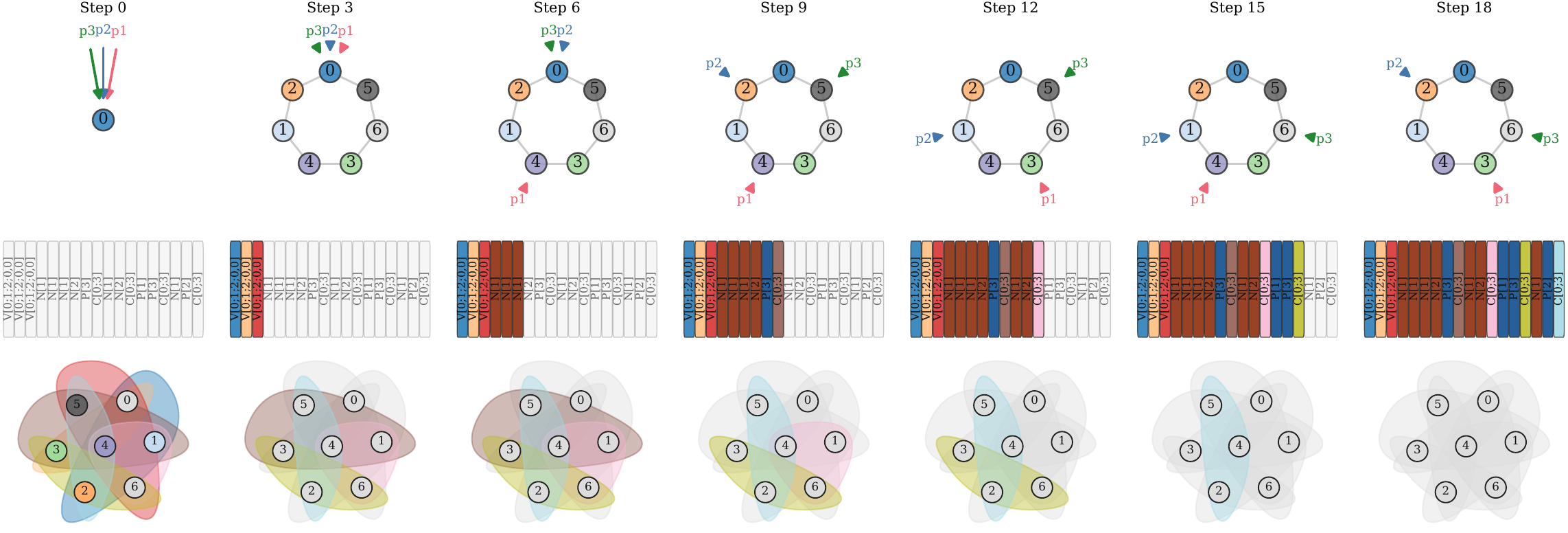}
  \caption{Greedy $\HTS$ encoding of the Fano plane STS(7), sampled
  every three instructions.  Top: the CDLL of the simulated machine with
  the pointers $\ptr{1}$, $\ptr{2}$, $\ptr{3}$.  Middle: the output
  tape, filling left to right as tokens are emitted.  Bottom: the input
  hypergraph; each of the seven hyperedges fades once an emitted
  $V$/$C$ token consumes it.  The final string has $18$ tokens: $7$
  insertion tokens ($3$ $V$, $4$ $C$) and $11$ movements.}
  \label{fig:h2s-steps}
\end{figure}

\begin{remark}[String-length decomposition]
\label{rem:length}
The greedy encoder emits exactly one insertion token per hyperedge and
no $W$, so for a hypergraph with $m$ hyperedges
\[
  |w| \;=\; m \;+\; \sum_{t} \operatorname{cost}(\delta_t),
\]
where the sum runs over the emission steps.  Only the traversal term
depends on the encoding order; minimising string length therefore
reduces to minimising total pointer travel, which is what cascade rule
C1 pursues locally.  The node count is recovered as
$|V| = 1 + \sum_{V\text{-tokens}} j$.  For the Fano plane,
$|w| = 7 + 11 = 18$.
\end{remark}

\subsection{Structural tuples and the canonical string}
\label{sec:canonical}

The greedy encoder is deterministic once the start node is fixed and the
C5 branches are resolved, but its output still depends on $v_0$.  The
canonical algorithm removes this dependence with a structural seed
selector.  For a node $v$ and depth $h \geq 1$, let
\[
  \xi_h(v) \;=\; \bigl|\{\, u \in V : d_H(u, v) = h \,\}\bigr|,
  \qquad
  \xi(v) = \bigl(\xi_1(v), \xi_2(v), \xi_3(v)\bigr),
\]
the number of nodes at primal distance exactly $h$ from $v$, collected
up to depth $3$; for a hyperedge $e$, let
$\eta(e) = \sum_{v \in e} \xi(v)$ componentwise.  Both tuples are
isomorphism-invariant because $\phi$ preserves primal distances.  The
depth $3$ is inherited from \IsalGraph{}; whether it must grow for
specific hypergraph families is an open question.

\begin{definition}[Canonical string]
\label{def:canonical}
Tokens are ordered by their kind rank $W < N < P < V < C$ and then by
their numeric parameters, and token sequences are compared
lexicographically under this order.  The \emph{canonical string} of a
connected hypergraph $H$ is
\[
  \wstar_H \;=\; \operatorname{lexmin}
  \bigl\{\, \HTS(H, v_0) \;:\;
     v_0 \in \arg\max\nolimits_{\mathrm{lex}} \xi(v) \,\bigr\},
\]
where each $\HTS$ run resolves its C5 ties by exploring all branches
and returning its lexicographically smallest completed string.
\end{definition}

Restricting the start nodes to the $\arg\max$ of $\xi$ is sound because
an isomorphism maps maximal-$\xi$ nodes to maximal-$\xi$ nodes, and it
narrows the seed set on irregular inputs; on
vertex-transitive hypergraphs, where every node attains the maximum, the
canonical algorithm degenerates to running the greedy from all $n$
nodes.

\begin{example}
\label{exa:fano}
For the Fano plane STS(7) with $k = 3$ the canonical algorithm returns
the $18$-token string
\[
  \wstar = \mathtt{V_{1,2}\,V_{1,2}\,V_{1,2}\,
  N_1\,N_1\,N_1\,N_2\,P_3\,C_3\,
  N_1\,N_2\,C_3\,
  P_1\,P_3\,C_3\,
  N_1\,P_2\,C_3} .
\]
The three $V_{1,2}$ tokens materialise the six non-seed nodes and the
three lines through the seed node; the four $C_3$ tokens complete the
remaining lines after repositioning the pointers.
Figures~\ref{fig:h2s-steps} and~\ref{fig:s2h-steps} trace the encoding
and the decoding of this string.
\end{example}

\replaced{We now state and prove the two properties on which the
isomorphism application rests.  Section~\ref{sec:results} reports the
matching empirical checks.}{We now state the two properties on which the isomorphism application
rests.  Both are conjectures: we argue them informally, validate them
empirically in Section~\ref{sec:results}, and defer the formal proofs to
a dedicated theoretical paper, exactly as \IsalGraph{} and \IsalSR{} did
for their canonical claims \citep{isalgraph2026, isalsr2026}.}

\begin{addedblock}
Both proofs were already complete when this preprint was first posted,
and the text of that version deferred them to a companion paper.  We
print them here so that the readers who have since sent us proofs of
these two statements can check their arguments against ours.

\paragraph{Encoder state.}
Fix a start node $v_0$.  A run of $\HTS$ holds a state
$\sigma = (\mu, C, \cdll, \ptr{1}, \dots, \ptr{k})$: a partial injection
$\mu$ from $V(H)$ to the machine's nodes with $\mu(v_0)$ the initial
node, the set $C \subseteq E(H)$ of consumed hyperedges, and the machine
CDLL and pointers.  For $e \in E(H)$ put
$M(e) = e \cap \operatorname{dom} \mu$ and
$U(e) = e \setminus \operatorname{dom} \mu$.  A waiting hyperedge $e$ is
\emph{insertable at $\delta$} when the displaced $\ptr{1}, \dots,
\ptr{i}$ rest on $\mu(M(e))$, where $i = |M(e)|$; it yields a $V_{i,j}$
candidate with $j = |U(e)| \geq 1$, or a $C_i$ candidate when
$U(e) = \varnothing$.  The state is \emph{terminal} when $C = E(H)$.

No step of the encoder reads a node or a hyperedge identifier: C1 reads
displacement costs, C2 the token kind, C3 the pair $(i, j)$, C4 the
structural tuple $\eta$, and C5 branches over \emph{every} candidate
that survives C1--C4.  This is the property both proofs below use.

\begin{lemma}[Progress]
\label{lem:progress}
Every reachable non-terminal state has an insertable hyperedge, and
every run halts after exactly $m = |E(H)|$ insertion tokens with
$\operatorname{dom} \mu = V(H)$ and $C = E(H)$.
\end{lemma}

\begin{proof}
Let $\sigma$ be reachable and non-terminal, and $r = |\operatorname{dom}
\mu| \geq 1$.  The CDLL holds $r$ slots, so any pointer reaches any slot
in at most $r$ moves and every displacement is available.

If some waiting $e$ has $U(e) = \varnothing$, then $|e| \leq k$ and the
displacement parking $\ptr{1}, \dots, \ptr{|e|}$ on the $|e|$ slots of
$\mu(e)$ makes $e$ insertable as a $C_{|e|}$ candidate.  Otherwise every
waiting hyperedge has an unmapped member, so
$\operatorname{dom} \mu \neq V(H)$.  By connectivity pick adjacent
$x \in \operatorname{dom}\mu$ and $y \notin \operatorname{dom}\mu$, and a
hyperedge $e$ containing both.  A consumed hyperedge has all its members
mapped, so $e$ is waiting; and $i = |M(e)| \geq 1$ with
$j = |U(e)| \geq 1$ and $i + j = |e| \leq k$ gives $i \leq k - 1$.
Parking $\ptr{1}, \dots, \ptr{i}$ on $\mu(M(e))$ makes $e$ insertable as
a $V_{i,j}$ candidate.

Each step consumes one hyperedge, so exactly $m$ insertion tokens are
emitted.  At $C = E(H)$ no node is unmapped, since an unmapped node
would lie in a hyperedge that is still waiting.
\end{proof}
\end{addedblock}

\begin{theorem}[Round-trip fidelity]
\label{conj:roundtrip}
For every connected hypergraph $H$ of arity at most $k$ and every start
node $v_0 \in V(H)$,
\(
  \STH\bigl(\HTS(H, v_0), k\bigr) \iso H .
\)
\end{theorem}

\begin{addedblock}
\begin{proof}
Fix a run and let $w_t$ be the string emitted after $t$ steps and
$\sigma_t = (\mu_t, C_t, \cdll_t, \ptr{})$ the state reached.  We claim
that the machine state of $\STH(w_t, k)$ has node set
$\mu_t(\operatorname{dom} \mu_t)$, hyperedge set
$\{\mu_t(e) : e \in C_t\}$ with $|C_t|$ distinct members, and CDLL and
pointers equal to $\cdll_t$ and $\ptr{}$.

At $t = 0$ both sides carry the single node $\mu_0(v_0)$, no hyperedge,
and all pointers on that node.  Assume the claim at $t$.  The emitted
movement block moves the machine pointers exactly as the encoder moves
its own, so the CDLL and the pointers agree after it.  If the step emits
$V_{i,j}$ for $e$, the token creates $j$ fresh nodes and splices them
after $\ptr{1}$, and the encoder extends $\mu$ by mapping the $j$
members of $U(e)$ to those nodes in the branched order; the created
hyperedge is
$\{\mathrm{val}(\ptr{1}), \dots, \mathrm{val}(\ptr{i})\} \cup
\{\text{new nodes}\} = \mu_{t+1}(M(e)) \cup \mu_{t+1}(U(e)) =
\mu_{t+1}(e)$.  If the step emits $C_i$ for $e$, the pointed nodes are
$\mu_t(e)$ by the choice of displacement, and the token creates
$\mu_t(e)$.  Either way $C_{t+1} = C_t \cup \{e\}$ and the claim holds
at $t + 1$.  The images stay distinct because $\mu_t$ is injective, so
distinct hyperedges of $H$ have distinct images and the no-op branch of
$C_i$ never fires.

By Lemma~\ref{lem:progress} the run halts with
$\operatorname{dom} \mu = V(H)$ and $C = E(H)$.  Then $\mu$ is a
bijection from $V(H)$ onto the node set of $\STH(w, k)$ carrying $E(H)$
onto its hyperedge set, so $\mu^{-1}$ is an isomorphism.
\end{proof}
\end{addedblock}

\deleted{The encoder maintains, by construction, a bijection between the
machine's nodes and the visited nodes of $H$, and every emitted
insertion token adds the image of exactly one waiting hyperedge under
that bijection.  If the loop invariant --- the machine state always
equals the state that $\STH$ reaches on the partial string --- holds
through every branch, the final correspondence is an isomorphism.
Establishing the invariant rigorously is the missing step.}

\begin{addedblock}
Completeness needs one more ingredient: the search must see the input
only through its incidence structure.  We make this precise by
transporting whole encoder states along an isomorphism.

\begin{definition}[Corresponding states]
\label{def:correspond}
Let $\phi : H_1 \to H_2$ be an isomorphism.  States $\sigma_1$ on $H_1$
and $\sigma_2$ on $H_2$ \emph{correspond under $\phi$} when
$\mu_2 = \mu_1 \circ \phi^{-1}$, $C_2 = \phi(C_1)$, and the two machine
states are equal.
\end{definition}

\begin{lemma}[Seed equivariance]
\label{lem:seed}
An isomorphism $\phi : H_1 \to H_2$ maps
$\arg\max_{\mathrm{lex}} \xi$ of $H_1$ onto that of $H_2$.
\end{lemma}

\begin{proof}
$\phi$ is an isomorphism of the primal graphs, hence preserves primal
distances, hence $\xi(\phi(v)) = \xi(v)$ for every node $v$.  The two
sets of tuples therefore agree, and $\phi$ carries maximisers to
maximisers.
\end{proof}

\begin{lemma}[Step equivariance]
\label{lem:step}
Let $\sigma_1$ and $\sigma_2$ correspond under $\phi$.  For every
displacement $\delta$, a waiting $e$ is insertable at $\delta$ in
$\sigma_1$ if and only if $\phi(e)$ is insertable at $\delta$ in
$\sigma_2$, and the two candidates agree on cost, kind, $(i, j)$ and
$\eta$.  Hence C1--C4 select the same displacement and the same
insertion token on both sides, their surviving candidate sets satisfy
$T_2 = \phi(T_1)$, and executing the branch at $e$ in $\sigma_1$ and at
$\phi(e)$ in $\sigma_2$ under the same assignment of new members to
insertion positions yields corresponding states.
\end{lemma}

\begin{proof}
The machine states are equal, so a pointer displaced by $\delta$ rests
on the same machine node on both sides; by $\mu_2 = \mu_1 \circ
\phi^{-1}$ the node it represents in $H_2$ is the $\phi$-image of the
node it represents in $H_1$.  Membership, $|M(e)|$, $|U(e)|$, arity and
waiting status are therefore transported by $\phi$, which gives the
equivalence and the equality of cost, kind and $(i, j)$; and
$\eta(\phi(e)) = \eta(e)$ because $\phi$ preserves primal distances.  So
C1--C4 read equal values, whence the same choice and
$T_2 = \phi(T_1)$.  In the branch both sides map their $j$ new members
to the same $j$ fresh machine nodes at the same positions, so
$\mu_2' = \mu_1' \circ \phi^{-1}$ and $C_2' = \phi(C_1')$: the successor
states correspond.
\end{proof}

\begin{lemma}[Equal completions]
\label{lem:complete}
Corresponding states return the same string.
\end{lemma}

\begin{proof}
Induction on $|E| - |C|$, which is equal on both sides.  At $0$ both
return the empty string.  Otherwise the emitted movement block and
insertion token are equal by Lemma~\ref{lem:step}.  C5 branches over the
surviving candidates and over the assignments of new members to
insertion positions; by Lemma~\ref{lem:step} these two branch sets are
in bijection through $\phi$, and corresponding branches lead to
corresponding successor states, whose completions are equal by the
induction hypothesis.  Both sides thus minimise over the same finite set
of strings, and the lexicographic minimum of a set is determined by the
set.
\end{proof}
\end{addedblock}

\begin{theorem}[Canonical completeness]
\label{conj:complete}
For all connected hypergraphs $H_1, H_2$ of arity at most $k$,
\(
  H_1 \iso H_2 \iff \wstar_{H_1} = \wstar_{H_2} .
\)
\end{theorem}

\begin{addedblock}
\begin{proof}
$(\Rightarrow)$  Let $\phi : H_1 \to H_2$ be an isomorphism.  By
Lemma~\ref{lem:seed} it maps the seed set of $H_1$ onto the seed set of
$H_2$.  For a seed $v_0$ the initial states of $\HTS(H_1, v_0)$ and
$\HTS(H_2, \phi(v_0))$ correspond under $\phi$, so
$\HTS(H_1, v_0) = \HTS(H_2, \phi(v_0))$ by Lemma~\ref{lem:complete}.
The two sets of candidate strings are therefore equal, and so are their
lexicographic minima.

$(\Leftarrow)$  Put $w = \wstar_{H_1} = \wstar_{H_2}$.  Then $w$ is the
output of a run of $\HTS$ on $H_1$ and of a run on $H_2$, so
Theorem~\ref{conj:roundtrip} gives $\STH(w, k) \iso H_1$ and
$\STH(w, k) \iso H_2$.  $\STH$ is deterministic, hence
$H_1 \iso H_2$.
\end{proof}

\begin{remark}[What the two directions cost]
\label{rem:directions}
$(\Leftarrow)$ uses only Theorem~\ref{conj:roundtrip}: whichever string a
tie-break returns, it decodes to a copy of the input, so equal strings
certify isomorphism under any tie-break rule.  $(\Rightarrow)$ is the
direction that needs C5 to branch over every candidate surviving
C1--C4.  A variant that resolves the residual tie by a node or hyperedge
identifier is not a function of the abstract hypergraph --- two
presentations of one hypergraph can then receive different strings ---
and for it $(\Rightarrow)$ is false.
\end{remark}

\begin{remark}[Node labels]
\label{rem:labelscope}
Both theorems are stated for the unlabelled inputs of this paper.  The
string never records the label of the start node, since $V$ tokens carry
only the labels of the nodes they create and $C$ tokens carry none.
Over a non-trivial vocabulary one therefore refines the seed rule to
$\arg\max_{\mathrm{lex}} (\xi(v), \ell(v))$ and compares the pair
$(\ell(v_0), \wstar_H)$, whose first component is then common to all
seeds; the proofs carry over unchanged.
\end{remark}
\end{addedblock}

\deleted{The forward direction would follow from the label-blindness of the
canonical search: an isomorphism $\phi$ maps the maximal-$\xi$ seeds of
$H_1$ onto those of $H_2$, preserves every quantity the cascade C1--C4
inspects (displacement costs depend only on the evolving CDLL geometry,
which is identical on both sides, and $\eta$ is
isomorphism-invariant), and maps the C5 branch sets onto each other; the
two hypergraphs then generate the same set of candidate strings and
hence the same lexicographic minimum.  The backward direction would
follow from Conjecture~\ref{conj:roundtrip} by transitivity:
$H_1 \iso \STH(\wstar_{H_1}) = \STH(\wstar_{H_2}) \iso H_2$.  What a
proof must establish rigorously is that the greedy-plus-backtracking
search depends only on the abstract incidence structure of the input,
never on node identifiers.  We leave this verification as future work;
Section~\ref{sec:results} reports the supporting evidence, with no
counterexample found.}

\begin{remark}[Hardness and cost]
\label{rem:hardness}
Hypergraph isomorphism is polynomial-time equivalent to graph
isomorphism, so \replaced{by}{under} \replaced{Theorem}{Conjecture}~\ref{conj:complete} computing
$\wstar_H$ is at least as hard as graph isomorphism, and no
polynomial-time algorithm should be expected.  The concrete search of
Definition~\ref{def:canonical} multiplies up to $n$ seed runs by the C5
branches, whose count is bounded by the product of $j!$ over the $V$
emissions of a run; we make no worst-case claim beyond this bound and
measure the cost empirically in Section~\ref{sec:results}.
\end{remark}

\begin{remark}[Parameters]
\label{rem:params}
Two canonical strings are comparable only when computed with the same
pointer count $k$; we use the maximal arity of the compared hypergraphs.
Disconnected inputs are rejected --- a per-component encoding with a
lexicographic merge is the natural extension but is not developed here.
\end{remark}

\section{Computational experiments\label{sec:experiments}}
The experiments answer two questions: does the round-trip property of
\replaced{Theorem}{Conjecture}~\ref{conj:roundtrip} hold in practice, and where does the
native canonical algorithm of Definition~\ref{def:canonical} stand,
in correctness and in cost, against the Levi-reduction baselines on
identical inputs.  \S\ref{sec:data} specifies the data cohort and
\S\ref{sec:setup} the measurement protocol.

\subsection{Data: connected uniform random hypergraphs}
\label{sec:data}

All instances are drawn from the $r$-uniform Erd\H{o}s--R\'enyi model:
given $n$ nodes and an inclusion probability $p$, each of the
$\binom{n}{r}$ candidate hyperedges of arity $r$ enters the hypergraph
independently with probability $p$ \citep{chodrow2020configuration}, as
implemented in the XGI library \citep{landry2023xgi}.  We parameterise
density by the expected number of hyperedges per node, $c$, setting
$p = c \, n / \binom{n}{r}$ so that $\mathbb{E}[m] = c \, n$, and we
report $c$ as the density axis.

Because the encoder requires connected input
(Remark~\ref{rem:params}) while the Levi baselines accept any input, a
raw Erd\H{o}s--R\'enyi sample would hand the two families different effective
workloads.  We therefore condition the generator on connectivity:
samples whose primal graph is disconnected are rejected and redrawn
under a deterministic seed walk ($\mathrm{seed},\ \mathrm{seed} +
1{,}000{,}003,\ \mathrm{seed} + 2 \cdot 1{,}000{,}003, \dots$), so every
method fingerprints exactly the same connected hypergraph and the
cohort remains reproducible from the seed list alone.  On the grid
below, acceptance typically requires one to three draws.  We call the resulting
distribution \emph{uniform Erd\H{o}s--R\'enyi conditional on
connectivity}.

The cohort sweeps three axes: node count
$n \in \{8, 12, 16, 20, 25\}$, arity $r = 3$, and density
$c \in \{1.0, 1.5, 2.0\}$, with ten seeds
($\mathrm{seed} \in \{0, \dots, 9\}$) per $(n, c)$ cell --- $150$
instances in total.  Figure~\ref{fig:cohort} shows how each axis
deforms the instances.  Arity $3$ is the regime of the classical
design-theory literature (Steiner triple systems are $3$-uniform), and
the $n$ and $c$ ranges were fixed in preliminary runs as the largest
grid on which every method, including \IsalHG{}, terminates within the
$600$\,s per-fingerprint budget on every instance; the ceiling is set
by \IsalHG{}, not by the baselines, a point
Section~\ref{sec:discussion} returns to.  For the correctness protocol,
each instance $H$ is paired with $\sigma(H)$ for a vertex permutation
$\sigma$ drawn uniformly from the cell's pinned generator, giving $150$
isomorphic pairs with a known certificate.

\begin{figure}[t]
  \centering
  \includegraphics[width=\textwidth]{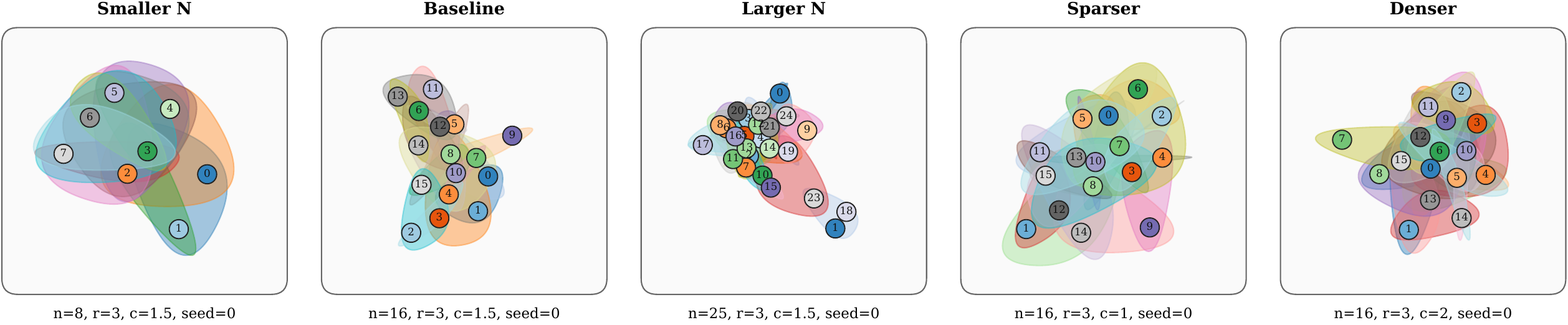}
  \caption{The cohort axes at seed $0$.  Centre-left: the baseline cell
  ($n = 16$, $r = 3$, $c = 1.5$).  Moving left and right varies the
  node count ($n = 8$ and $n = 25$ at the same density); the two
  right-hand panels vary the density at $n = 16$ ($c = 1$ and $c = 2$).
  Hyperedges are drawn as shaded regions over the numbered nodes.  All
  five instances are connected by construction under the
  reject-resample policy of \S\ref{sec:data}.}
  \label{fig:cohort}
\end{figure}

\subsection{Experimental setup}
\label{sec:setup}

\paragraph{Methods under test.}
Four methods compute one fingerprint per hypergraph and decide
isomorphism by fingerprint equality.  \IsalHG{} computes the canonical
string of Definition~\ref{def:canonical}; the search is implemented in
C++17 behind a Python interface.  The three baselines translate $H$
into its Levi graph $\levi(H)$ with the standard $2$-colouring
(node-vertices versus hyperedge-vertices) and canonically label it with,
respectively, nauty~2.8.8 through the \texttt{pynauty} binding,
bliss~0.77 through \texttt{python-igraph}, and Traces through the
\texttt{dreadnaut} interface of the nauty~2.9.1 distribution, whose
canonical graph line serves as the fingerprint.

The comparison is exact-versus-exact: a baseline must solve the same
problem \IsalHG{} solves: an isomorphism decision with zero false
positives and zero false negatives on every pair, while existing
as runnable software.  These two requirements pin the baseline set
exactly.  The Levi reduction driven by nauty, Traces, or bliss is, as
discussed in \S\ref{sec:rel-exact-tools}, the only exact hypergraph
isomorphism procedure with public implementations, and design-theory
systems (SageMath, GAP) delegate to nauty internally, so benchmarking
the three engines directly also covers the tools built on them.  The
theoretical exact algorithms of \S\ref{sec:rel-exact-theory} have no
implementations to run.  The Weisfeiler--Leman family of
\S\ref{sec:rel-wl} is excluded on principle rather than availability:
those methods compute incomplete invariants, so a runtime comparison
against them would compare answers to different questions; their own
published counterexamples already locate the correctness difference.

\paragraph{Measurements.}
For each of the $600$ (method, instance) pairs we record the
fingerprint wall-clock time as the median of $10$ repeated calls
(\texttt{time.perf\_counter}), its interquartile range, the peak
resident-set increment of the call
(\texttt{getrusage(RUSAGE\_SELF)}), and the fingerprint length in
bytes.  A $600$\,s watchdog per fingerprint converts non-termination
into an explicit DNF record.  Because Traces runs as a subprocess, its
memory column measures the parent process only; its wall-clock
includes the subprocess start-up.  Both caveats are flagged where they
matter.

\paragraph{Correctness protocol.}
Three checks run on top of the measurements.  (i)~\emph{Positive
pairs}: every method must return \texttt{true} on
$(H, \sigma(H))$ for each of the $150$ instances --- $600$ checks in
total.  (ii)~\emph{Cross-method agreement}: within every $(n, c)$ cell,
the iso-equivalence partition that each method's fingerprints induce on
the cell's ten instances must coincide across the four methods.
(iii)~\emph{Round-trip}: for every cohort instance and for five named
designs (the Fano plane STS(7), STS(9), the two non-isomorphic
\replaced{cyclic triple orbits $C_{13}(0,1,4)$ and $C_{13}(0,1,6)$,
each a \emph{partial} triple system of $13$ blocks on $13$ points
rather than a Steiner triple system}{cyclic
STS(13) with starter blocks $\{0,1,4\}$ and $\{0,1,6\}$}, and the
generalized quadrangle GQ(2,2)), the canonical string is decoded with
$\STH$ and the result is compared against the original hypergraph with
nauty as the independent oracle.

\paragraph{Execution.}
The timing sweep ran on the CPU partition of the Picasso supercomputer
(SCBI, University of M\'alaga) as one single-core SLURM array task per
(method, instance) pair, $600$ tasks in total, with $16$\,GB of memory
per task; every task completed.  The C++ engine was compiled with GCC
at \texttt{-O3}.  The round-trip checks, which assert a boolean rather
than a timing, ran on a development workstation.  All seeds are pinned,
and the cohort, the configuration, and the per-cell result records are
archived with the project.

\section{Results\label{sec:results}}
We report the round-trip verification (\S\ref{sec:res-roundtrip}), the
correctness agreement across the four methods
(\S\ref{sec:res-agreement}), the wall-clock and memory characterisation
(\S\ref{sec:res-runtime}), and the fingerprint lengths
(\S\ref{sec:res-fplen}).

\subsection{Round-trip verification}
\label{sec:res-roundtrip}

The round-trip check passes on every tested input: all $150$ cohort
instances and all five named designs satisfy
$\STH(\wstar_H, k) \iso H$, with the isomorphism confirmed by nauty on
the Levi graphs (Table~\ref{tab:roundtrip}).  Canonical strings range
from $98$ to $1{,}556$ serialised bytes over the cohort (median $174$
bytes at $n = 8$, $982$ bytes at $n = 25$).
Figures~\ref{fig:h2s-steps} and~\ref{fig:s2h-steps} display one such
round trip on the Fano plane: the encoder produces the $18$-token canonical
string of Example~\ref{exa:fano}, and $\STH$ decodes it back into a
hypergraph isomorphic to the input.  These results \replaced{confirm
Theorem}{support
Conjecture}~\ref{conj:roundtrip} on $155$ structurally diverse inputs\replaced{,
and check the implementation against the proof}{;
they do not prove it}.

\begin{table}[t]
  \centering
  \caption{Round-trip verification.  Each input is encoded to its
  canonical string $\wstar_H$, decoded with $\STH$, and compared
  against the original with nauty as the independent oracle.  The two
  \replaced{cyclic triple orbits $C_{13}$}{cyclic STS(13)} are not isomorphic to each other and receive distinct
  canonical strings.\added{  $C_{13}(b)$ is the orbit of the starter
  block $b$ under $x \mapsto x + 1 \bmod 13$: $13$ blocks covering $39$
  of the $78$ point pairs, so it is a \emph{partial} triple system, not
  a Steiner triple system (which would need $26$ blocks).}}
  \label{tab:roundtrip}
  \begin{tabular}{@{}lrrcr@{}}
    \toprule
    Input & $n$ & $m$ & $\STH(\wstar_H) \iso H$ & $|\wstar_H|$ (bytes) \\
    \midrule
    Connected uniform ER cohort & $8$--$25$ & $6$--$67$ & $150/150$ & $98$--$1{,}556$ \\
    Fano plane STS(7)           & $7$  & $7$  & yes & $121$ \\
    STS(9)                      & $9$  & $12$ & yes & $227$ \\
    \replaced{$C_{13}(0,1,4)$}{Cyclic STS(13), starter $\{0,1,4\}$} & $13$ & $13$ & yes & $258$ \\
    \replaced{$C_{13}(0,1,6)$}{Cyclic STS(13), starter $\{0,1,6\}$} & $13$ & $13$ & yes & $263$ \\
    GQ(2,2)                     & $15$ & $15$ & yes & $278$ \\
    \bottomrule
  \end{tabular}
\end{table}

\subsection{Correctness agreement}
\label{sec:res-agreement}

All $600$ positive-pair checks pass: on every one of the $150$
$(H, \sigma(H))$ pairs, all four methods --- \IsalHG{}, nauty, Traces,
and bliss --- confirm isomorphism on every pair, with zero failures.  Within every $(n, c)$ cell, the four methods also induce
identical iso-equivalence partitions over the cell's ten instances, so
no pair exists anywhere in the cohort on which \IsalHG{} and a Levi
baseline disagree, in either direction.  The named designs add a
discrimination check at the other pole of the symmetry spectrum: the
two non-isomorphic \replaced{cyclic triple orbits $C_{13}(0,1,4)$ and
$C_{13}(0,1,6)$, which are both $3$-uniform, $3$-regular and
vertex-transitive and therefore agree}{cyclic STS(13), which agree} on the trivial
invariants ($n$, $m$, degree and arity sequences), receive distinct
canonical strings (Table~\ref{tab:roundtrip}).  Across everything we
tested, \replaced{the verdicts match Theorem~\ref{conj:complete}
exactly}{no counterexample to Conjecture~\ref{conj:complete} appeared}.
We note the limits of this evidence: the cohort is small
($n \leq 25$, $r = 3$), and random hypergraphs are rigid with high
probability, so the hard cases for canonical algorithms ---
large-automorphism designs at scale --- are only touched by the five
fixtures of Table~\ref{tab:roundtrip}.

\subsection{Runtime and memory}
\label{sec:res-runtime}

Table~\ref{tab:runtime} and Figure~\ref{fig:wallclock} report
wall-clock time; every cell of the grid completed for every method.  The Levi baselines dominate
the entire grid.  nauty is the fastest method on every cell, between
$0.048$ and $0.201$\,ms per fingerprint; bliss runs within a
factor of $2$--$4$ of nauty; Traces stays at a flat $\approx 3$\,ms,
which its fitted exponent ($n^{-0.02}$) identifies as the \texttt{dreadnaut}
subprocess start-up floor rather than algorithmic cost.  \IsalHG{}
spans $16.9$\,ms to $22.8$\,s over the same grid.  The fitted growth
exponents separate the regimes: $n^{5.79}$ for \IsalHG{} (pooled over
$c$; $n^{6.2}$ at $c = 2$ alone) against $n^{0.93}$ for nauty and
$n^{0.80}$ for bliss.  The per-cell geometric-mean ratio of \IsalHG{}
to the best Levi engine grows from $311\times$ at the smallest, densest
cell to $117{,}672\times$ at $(n, c) = (25, 1.5)$, roughly tripling
with each step in $n$.  No crossover regime appears anywhere in the
tested grid, and the exponent gap implies none will appear at larger
$n$ under the current canonical search.

\begin{table}[t]
  \centering
  \caption{Median wall-clock per fingerprint (ms; median over ten
  seeds, ten repeats per seed) for the fifteen $(n, c)$ cells of the
  $r = 3$ grid, and the geometric-mean ratio of \IsalHG{} over the
  best Levi engine per instance.  No method timed out on any instance
  ($600$\,s budget).}
  \label{tab:runtime}
  \begin{tabular}{@{}rr rrrr r@{}}
    \toprule
    $n$ & $c$ & \IsalHG{} & Levi (nauty) & Levi (bliss) & Levi (Traces) & ratio \\
    \midrule
     $8$ & $1.0$ & $16.9$     & $0.048$ & $0.12$ & $3.22$ & $315\times$ \\
     $8$ & $1.5$ & $21.2$     & $0.059$ & $0.14$ & $3.27$ & $345\times$ \\
     $8$ & $2.0$ & $22.3$     & $0.073$ & $0.16$ & $2.91$ & $311\times$ \\
    $12$ & $1.0$ & $42.7$     & $0.071$ & $0.17$ & $3.31$ & $548\times$ \\
    $12$ & $1.5$ & $44.9$     & $0.078$ & $0.18$ & $3.08$ & $541\times$ \\
    $12$ & $2.0$ & $55.9$     & $0.100$ & $0.24$ & $3.30$ & $592\times$ \\
    $16$ & $1.0$ & $176.7$    & $0.088$ & $0.20$ & $3.20$ & $1{,}985\times$ \\
    $16$ & $1.5$ & $241.5$    & $0.111$ & $0.23$ & $3.06$ & $2{,}141\times$ \\
    $16$ & $2.0$ & $492.3$    & $0.129$ & $0.29$ & $3.46$ & $4{,}319\times$ \\
    $20$ & $1.0$ & $822.0$    & $0.099$ & $0.20$ & $2.72$ & $7{,}824\times$ \\
    $20$ & $1.5$ & $2{,}039.7$  & $0.138$ & $0.32$ & $3.37$ & $13{,}128\times$ \\
    $20$ & $2.0$ & $3{,}037.3$  & $0.162$ & $0.35$ & $3.17$ & $15{,}934\times$ \\
    $25$ & $1.0$ & $6{,}863.5$  & $0.145$ & $0.30$ & $3.21$ & $57{,}223\times$ \\
    $25$ & $1.5$ & $19{,}170.6$ & $0.176$ & $0.34$ & $2.92$ & $117{,}672\times$ \\
    $25$ & $2.0$ & $22{,}786.1$ & $0.201$ & $0.44$ & $3.17$ & $110{,}719\times$ \\
    \bottomrule
  \end{tabular}
\end{table}

\begin{figure}[t]
  \centering
  \includegraphics[width=0.92\textwidth]{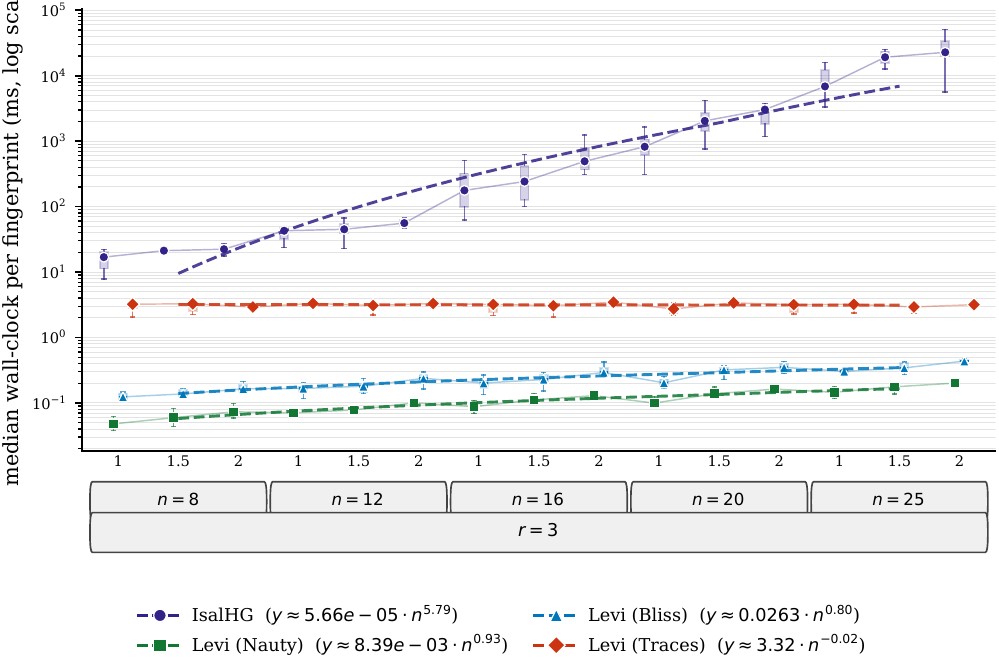}
  \caption{Median wall-clock per fingerprint (log scale) across the
  fifteen $(n, c)$ cells, grouped by $n$ with $c \in \{1, 1.5, 2\}$
  inside each group.  Markers are medians over ten seeds with
  interquartile ranges; dashed lines are per-method power-law fits
  $T \approx a \cdot n^{\beta}$ with the fitted exponents in the
  legend.  The flat Traces line is the subprocess start-up floor, not
  the algorithm.}
  \label{fig:wallclock}
\end{figure}

Memory does not separate the methods at this scale
(Figure~\ref{fig:memory}): the peak resident-set increment per
fingerprint stays below $2$\,MiB for every method on every cell, with
quantisation at the page-size granularity dominating the visible
variation.  The Traces column measures the parent process only
(\S\ref{sec:setup}), so its flat $0.4$\,MiB line excludes the
subprocess.

\begin{figure}[t]
  \centering
  \includegraphics[width=0.92\textwidth]{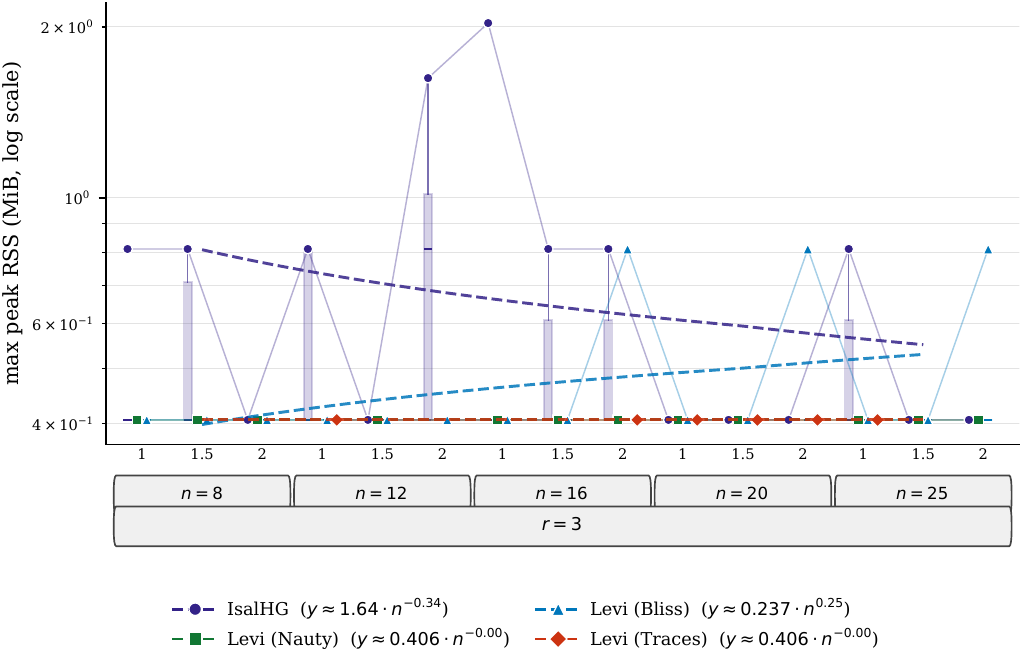}
  \caption{Peak resident-set increment per fingerprint (log scale)
  over the same grid as Figure~\ref{fig:wallclock}.  All methods stay
  below $2$\,MiB; page-granularity quantisation dominates the visible
  variation.  The Traces series measures the parent process only.}
  \label{fig:memory}
\end{figure}

\subsection{Fingerprint length}
\label{sec:res-fplen}

Figure~\ref{fig:fplen} compares the fingerprint sizes.  The canonical
string is comparable in length to the serialised canonical labellings of
the Levi engines: its median is $460$ bytes over the grid,
between nauty ($320$ bytes) and bliss ($784$ bytes), while Traces'
compact graph-line format is the smallest at $88$ bytes.  The
fitted growth is $n^{1.49}$ for \IsalHG{} against $n^{1.11}$ to
$n^{1.27}$ for the baselines --- consistent with
Remark~\ref{rem:length}: the string pays one token per hyperedge
($m = c \, n$ here) plus the pointer travel, which grows with the
CDLL length.  Fingerprint length is the one axis on which \IsalHG{} does not
trail the Levi baselines by orders of magnitude.

\begin{figure}[t]
  \centering
  \includegraphics[width=0.92\textwidth]{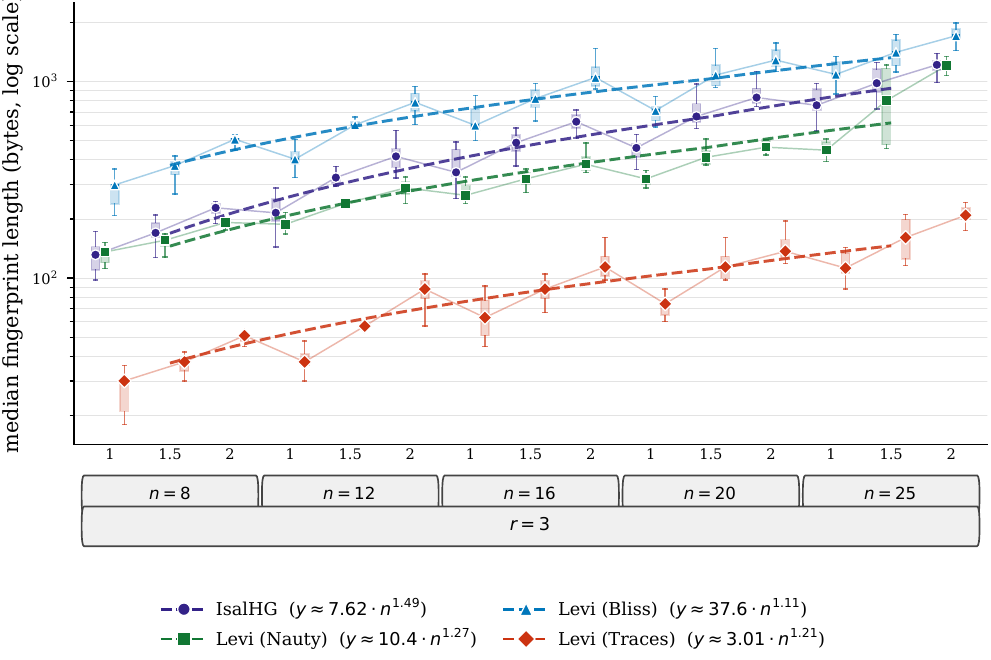}
  \caption{Median fingerprint length in bytes (log scale) over the
  same grid as Figure~\ref{fig:wallclock}, with per-method power-law
  fits in the legend.  \IsalHG{}'s canonical string lies between the
  serialised canonical labellings of nauty and bliss; Traces' graph
  line is the most compact.}
  \label{fig:fplen}
\end{figure}

\section{Discussion\label{sec:discussion}}
\paragraph{What the agreement establishes, and what it does not.}
The correctness outcome is uniform: $600$ positive-pair verdicts with
zero failures, identical partitions in every cell, a perfect round-trip
record, and distinct canonical strings for the two non-isomorphic
\replaced{cyclic triple orbits $C_{13}$}{STS(13)}.  Everything we measured is consistent with
\replaced{Theorems}{Conjectures}~\ref{conj:roundtrip} and~\ref{conj:complete}, and the
experiments were designed so that a single disagreement with any of the
three Levi engines would have surfaced as an explicit counterexample.
None did.  The evidence remains bounded: the cohort is small and, being
random, rigid with high probability, so the regime where canonical
algorithms historically fail --- large automorphism groups at scale ---
is represented only by the five designs of Table~\ref{tab:roundtrip}.
We leave the extension of the correctness campaign to the published
design catalogues for a separate empirical paper; formal proofs belong
to a theoretical one.

\paragraph{Reading the runtime gap.}
The gap is large and it grows: $311\times$ at $(n, c) = (8, 1)$,
$117{,}672\times$ at $(25, 1.5)$, with fitted exponents $n^{5.79}$
against $n^{0.93}$ (Figure~\ref{fig:wallclock}).  Two properties of the
canonical search of Definition~\ref{def:canonical} account for the
shape.  First, the seed rule runs the greedy once per maximal-$\xi$
node, and on near-regular random hypergraphs the maximal-$\xi$ class is
large, so the multiplier approaches $n$.  Second, every $V$ emission
branches over the assignments of its new nodes (cascade rule C5), and
the branch count compounds along the string, bounded by the product of
$j!$ over the $V$ emissions (Remark~\ref{rem:hardness}).  Both costs
are intrinsic to the search discipline, not to its implementation: the
measured engine is compiled C++, and the exponent gap means no constant
factor will close it.  During development we also evaluated several
encoder variants --- a single-seed greedy, colour-refinement pruning of
the branch sets, and an exhaustive search --- and none of them changed
the competitive picture on this cohort.

\paragraph{Fairness in both directions.}
The comparison flatters neither side.  nauty, Traces, and bliss embody
four decades of individualisation--refinement engineering
\citep{mckay1981practical, mckay2014practical, junttila2007bliss},
whereas the canonical algorithm measured here is the first
implementation of a new representation; we did not expect it to
match mature engines.  In the other direction, two
measurement artefacts favour neither side but bear noting: the
flat $3$\,ms Traces line is subprocess start-up rather than solver
time, so Traces would track the other engines if invoked in-process,
and the memory columns are uninformative at this scale (all methods
below $2$\,MiB), with the Traces value measuring the parent process
only.

\paragraph{Paths forward.}
The measurements identify the two components that dominate the cost.
The seed multiplier requires stronger isomorphism-invariant seed
selectors than depth-3 $\xi$ --- the design-theory practice of
pre-partitioning by configuration counts \citep{kaski2004steiner} is
the natural template.  The branch compounding requires a search
discipline that prunes equivalent branches instead of enumerating
them, in the spirit of the individualisation--refinement tree
\citep{mckay2014practical} and of the canonisation framework of
\citet{schweitzer2019unifying}, while keeping the canonical string as
the output object.  Both directions preserve the contribution of this
paper --- the representation and \replaced{its proved completeness}{its conjectured completeness} --- and
replace only the search that computes $\wstar_H$.\added{  A faster
search must keep the C5 branch set intact, since
Remark~\ref{rem:directions} shows that a tie broken by a node or
hyperedge identifier loses completeness.}

\paragraph{Scope limits.}
Four threats to validity delimit the claims.  The grid fixes $r = 3$
and $n \leq 25$, so nothing here speaks to higher arities or larger
inputs except the fitted exponents.  The cohort is conditioned on
connectivity, a distribution choice made so that all methods receive
identical inputs; raw Erd\H{o}s--R\'enyi sampling would instead route
disconnected instances away from \IsalHG{} (Remark~\ref{rem:params}).
The timing hardware is a single CPU-partition node class, and the
$600$\,s budget, although never reached, caps what the grid could have
explored.  Finally, fingerprint length (Figure~\ref{fig:fplen}) is the
only axis where the native encoding is presently competitive; we
report it as an observation, not as a claim of advantage.

\section{Conclusion\label{sec:conclusion}}
\paragraph{Summary of contributions.}
This paper has introduced \IsalHG{}, a sequential instruction-based
representation of finite connected hypergraphs of bounded arity.  A virtual machine defines the encoding: it manipulates a circular
doubly-linked list of node references through $k$ traversal pointers,
inserting hyperedges --- and, through the $V_{i,j}$ tokens, new nodes
--- as instructions execute.  We establish four points:
\begin{enumerate}[label=(\roman*)]
  \item \emph{A native representation with a closed language.}  Every
        string over $\SigHG$ decodes to a valid hypergraph
        (Proposition~\ref{prop:closure}), and the alphabet size grows
        only quadratically in the arity bound,
        $|\SigHG| = k(k-1)/2 + 3k + 1$.  This is, to our knowledge, the
        first executable-string representation of hypergraphs intended
        as a canonical isomorphism invariant.
  \item \emph{Round-trip fidelity, \replaced{proved and verified}{verified}.}  The greedy $\HTS$
        encoder and the $\STH$ interpreter invert each other up to
        isomorphism\added{ (Theorem~\ref{conj:roundtrip})} on every tested input: $150$ connected random
        uniform hypergraphs and five named designs, confirmed
        independently by nauty (\replaced{Table}{Conjecture~\ref{conj:roundtrip};
        Table}~\ref{tab:roundtrip}).
  \item \emph{Canonical completeness, \replaced{proved}{conjectured and unrefuted}.}  \added{The canonical
        string is a complete isomorphism invariant
        (Theorem~\ref{conj:complete}): equal strings certify isomorphism
        by round-trip, and isomorphic inputs drive the search through
        matching branches.}  The
        canonical string $\wstar_H$ agreed with nauty, Traces, and
        bliss on all $600$ isomorphism verdicts of the cohort and
        separated the two non-isomorphic \replaced{cyclic triple orbits
        $C_{13}$}{STS(13)}\deleted{
        (Conjecture~\ref{conj:complete}); no counterexample was found
        anywhere in the campaign}.
  \item \emph{An honest cost characterisation.}  On the tested grid the
        Levi baselines dominate wall-clock by three to five orders of
        magnitude (geometric-mean ratio $311\times$ to
        $117{,}672\times$), with fitted growth $n^{5.79}$ for the
        current canonical search against $n^{0.93}$ for nauty.
        Fingerprint length is the one axis where the native encoding
        is already comparable ($460$ versus $320$--$784$ median
        bytes).
\end{enumerate}

\paragraph{Limitations.}
\deleted{The two central conjectures --- round-trip fidelity
(Conjecture~\ref{conj:roundtrip}) and canonical completeness
(Conjecture~\ref{conj:complete}) --- remain unproven; we defer their proofs to a dedicated theoretical paper.}  The canonical search is
not competitive in runtime with the Levi pipeline anywhere in the
tested regime, and its fitted exponent implies the gap widens with
$n$.  The encoder requires connected input, hyperedge arities are
capped by the pointer count $k$, and the empirical evidence covers
$3$-uniform hypergraphs with $n \leq 25$ plus five design fixtures.

\paragraph{Future directions.}
\replaced{Our immediate priority is a}{Our immediate priorities are the formal proofs of both conjectures and a}
redesigned canonical search --- stronger structural seed invariants
and pruning of \replaced{branches that Theorem~\ref{conj:complete} shows
to be equivalent}{equivalent branches} --- that keeps the string as the
output object.  Our planned empirical campaign extends the correctness
evidence to the exhaustive Steiner-system catalogues, hard symmetric
families, and real-world hypergraph corpora, where the deduplication
workload also exercises the representation at scale.  On the
representation side, the labelled alphabet of Remark~\ref{rem:labels}
and the treatment of disconnected inputs complete the object class,
and the string form itself --- compact, closed, and sequential ---
suits the similarity-search and generative uses that \IsalGraph{}
and \IsalSR{} explore for their domains \citep{isalgraph2026,
isalsr2026}.

\section*{Acknowledgment}
The authors thankfully acknowledge the computer resources (Picasso
Supercomputer), technical expertise, and assistance provided by the
SCBI (Supercomputing and Bioinformatics) center of the University
of M\'alaga.

\bibliographystyle{unsrtnat}
\bibliography{references}

\end{document}